\documentclass{article}

\usepackage{latexsym}
\usepackage{bm} 
\usepackage{graphicx}
\usepackage[cmex10]{amsmath}
\usepackage{amssymb,amsthm,amsfonts}
\usepackage{url}
\usepackage{color}
\usepackage{algorithmic} 
\usepackage{authblk}

\newtheorem{thm}{Theorem}

\newtheorem{lemma}{Lemma}

\newtheorem{defn}{Definition}

\newcommand{\sH}{{\mathcal H}}

\newcommand{\zbar}{\bar{z}}
\newcommand{\sI}{\mathcal I}

\newcommand{\data}{\left\{ x_1,\dots,x_n \right\}}

\newcommand{\bbR}{\mathbb{R}} 
\newcommand{\uno}{\bm{1}} 
\newcommand{\rd}{\bbR^d}
\newcommand{\kernel}{\phi:\rd\times\rd\rightarrow\bbR}

\newcommand{\ind}[1]{\bm{1}_{\{#1\}}}

\newcommand{\brac}[1]{\left[#1\right]}
\newcommand{\set}[1]{\left\{#1\right\}}
\newcommand{\abs}[1]{\left\lvert #1 \right\rvert}
\newcommand{\paren}[1]{\left(#1\right)}
\newcommand{\norm}[1]{\left\|#1\right\|}
\newcommand{\inpr}[1]{\left<#1\right>}

\begin{document}

\title{Sparse Approximation of a Kernel Mean}
\author{Efr\'en Cruz Cort\'es}
\affil{
\small Department of Electrical Engineering and Computer Science \\
\small University of Michigan, Ann Arbor, MI}
\author{Clayton Scott}
\affil{
\small Department of Electrical Engineering and Computer Science \\
\small University of Michigan, Ann Arbor, MI}

\date{February 2015}
\maketitle

\begin{abstract}
Kernel means are frequently used to represent probability distributions in machine learning problems. In particular, the well known kernel density estimator and the kernel mean embedding both have the form of a kernel mean. Unfortunately, kernel means are faced with scalability issues. A single point evaluation of the kernel density estimator, for example, requires a computation time linear in the training sample size. To address this challenge, we present a method to efficiently construct a sparse approximation of a kernel mean. We do so by first establishing an incoherence-based bound on the approximation error, and then noticing that, for the case of radial kernels, the bound can be minimized by solving the $k$-center problem. The outcome is a linear time construction of a sparse kernel mean, which also lends itself naturally to an automatic sparsity selection scheme. We show the computational gains of our method by looking at three problems involving kernel means: Euclidean embedding of distributions, class proportion estimation, and clustering using the mean-shift algorithm.
\end{abstract}


\section{Introduction}

A {\em kernel mean} is a quantity of the form
\begin{equation}
\label{eqn:km}
\frac1{n} \sum_{i=1}^n \phi( \cdot, x_i),
\end{equation}
where $\phi$ is a {\em kernel} and $x_1, \ldots, x_n \in \rd$ are data
points. We define kernels rigorously below. Our treatment includes
many common examples of kernels, such as the Gaussian kernel, and
encompasses both symmetric positive definite kernels and kernels used for
nonparametric density
estimation.

Kernel means arise frequently in machine learning and nonparametric
statistics as representations of probability distributions. In this
context, $x_1, \ldots, x_n$ are understood to be realizations of some
unknown probability distribution. The kernel density estimator (KDE) is a
kernel mean that estimates the density of the data. The kernel mean
embedding (KME) is a kernel mean that maps the probability distribution
into a reproducing kernel Hilbert space. These two motivating applications
of kernel means are reviewed in more detail below.

This work is concerned with efficient computation of a sparse approximation of a kernel mean, taking the form
\begin{equation}
\label{eqn:skm}
\sum_{i=1}^n \alpha_i \phi(\cdot, x_i)
\end{equation}
where $\alpha_i \in \bbR$ and $k := | \{ i \, : \, \alpha_i \ne 0\} | \ll n$. In other words,
given $x_1, \ldots, x_n$, a kernel $\phi$, and a target sparsity $k$, we
seek a sparse kernel mean \eqref{eqn:skm} that accurately approximates the
kernel mean \eqref{eqn:km}. This problem is motivated by applications
where $n$ is so large that evaluation or manipulation of the full kernel
mean is computationally prohibitive. A sparse kernel mean can be evaluated
or manipulated much more efficiently. In the large $n$ regime, the
sparse approximation algorithm itself must be scalable, and as we argue
below, existing sparse approximation strategies are too slow.

Our primary contribution is an efficient algorithm for sparsely
approximating a kernel mean. The algorithm results from minimizing a
sparse approximation bound based on a novel notion of incoherence. We show
that in the context of kernel means based on a {\em radial} kernel (defined
below), minimizing the sparse approximation bound is equivalent to solving
the $k$-center problem on $x_1, \ldots, x_n$, which in turn leads to an
efficient algorithm.

The rest of the paper is outlined as follows. In Section \ref{sec:motivation} we review the KDE and KME, which motivate this work, and also introduce a general definition of kernel that encompasses both of these settings. Next, in Section \ref{sec:abstractformulation} we formulate the problem of sparsely approximating a sample mean in an inner product space, followed by a review of related work in Section \ref{sec:relatedwork}, where we also detail our contributions. In Section \ref{sec:subsetselection} we establish an incoherence-based sparse approximation bound. We then use the principle of bound minimization in Section \ref{sec:skm} to derive a scalable algorithm for sparse approximation of kernel means, with a sparsity auto-selection scheme presented in Section \ref{sec:autoselect}. Finally, Section \ref{sec:experiments} applies our methodology in three different machine learning problems that rely on large-scale KDEs and KMEs, and demonstrates the efficacy of our approach. A preliminary version of this work appeared in \cite{cruz2014scalable}. A Matlab implementation of our algorithm is available at \cite{scott0000website}.

\section{Motivating Applications} \label{sec:motivation}

Our work is motivated by two primary examples of kernel means. We review
the KDE and KME separately, and then propose a general notion of kernel
that encompasses the essential features of both settings and is
sufficient for addressing the sparse approximation problem. By way of
notation, we denote $\brac{n} := \{1, \ldots, n\}$.

\subsection{Kernel Density Estimation}

Let $\data\subset \rd$ be a random sample from a distribution with density $f$. In the context of kernel density estimation, a kernel is a function $\phi$ such that for all $x'$,
$\int \phi(x,x') dx = 1$. In addition, $\phi$ is sometimes also chosen to
be nonnegative, although this is not necessary for theoretical properties such as consistency. The kernel density estimator of $f$ is the function
$$
\widehat{f} = \frac{1}{n}\sum_{i\in\brac{n}}{\phi(\cdot,x_i)}.
$$

The KDE is used as an ingredient in a number of machine learning methodologies. For example, a common approach to classification is a plug-in rule that estimates the class-conditional densities with separate KDEs \cite{titterington1981comparison, hand1983comparison}. In anomaly detection, a detector of the form $\widehat{f}(x) \substack{> \\ <} \gamma$ is commonly employed to determine if a new realization comes from $f$ \cite{desforges1998applications, yeung2002parzen,markou2003novelty,chandola2009anomaly}. In clustering, the mean-shift algorithm forms a KDE and associates each data point to the mode of the KDE that is reached by hill-climbing \cite{cheng1995mean}.

Evaluating the KDE at a single test point requires $O(n)$ kernel
evaluations, which is undesirable and perhaps prohibitive for large $n$.
On the other hand, a sparse approximation with sparsity $k$ requires only $O(k)$ kernel
evaluations. This problem is magnified in algorithms such as mean-shift,
where a (derivative of a) KDE is evaluated numerous times for each data
point. In our experiments below, we demonstrate the computational savings
of our approach in KDE-based algorithms for the embedding of probability
distributions and mean-shift clustering.

\subsection{Kernel Mean Embedding of Distributions}

Let $\data\subset \rd$ be a random sample from a distribution $P$. A {\em symmetric positive definite kernel} is a function $\kernel$ that is symmetric and is such that all square matrices of the form $[\phi(x_i, x_j)]_{i,j=1}^n$ are positive semidefinite. Every symmetric positive definite kernel is associated to a unique Hilbert space of functions called a reproducing kernel Hilbert space (RKHS), which can be thought of as the closed linear span of $\{\phi(\cdot,x) \, | \, x \in \rd\}$ \cite{steinwart08}. The RKHS has a property known as the {\em reproducing property} which states that for all $f$ in the RKHS, $f(x) = \langle f, \phi(\cdot, x) \rangle$.

The idea behind the kernel mean embedding is to select a symmetric positive definite kernel $\phi$, and embed $P$ in the RKHS associated with $\phi$ via the mapping
$$
\Psi(P):=\int \phi(\cdot,x)dP(x) .
$$
Since $P$ is unknown, this mapping is estimated via the kernel mean
$$
\widehat{\Psi}(P) := \frac{1}{n}\sum_{i\in\brac{n}}{\phi(\cdot,x_i)}.
$$
The utility of the KME derives from the fact that for certain kernels,
$\Psi$ is injective. This permits the treatment of probability
distributions as objects in a Hilbert space, which allows many existing
machine learning methods to be applied in problems where probability
distributions play the role of feature vectors \cite{smola2007hilbert,
gretton2012kernel, fukumizu2011kernel,gurram2012contextual}. For example,
suppose that random samples of size $n$ are available from several
probability distributions $P_1,\ldots,P_N$. A KME-based algorithm will require the computation of all
pairs of inner products of kernel mean embeddings of these distributions.
If $x_1, \dots, x_n \sim P$ and $x_1', \ldots, x_n' \sim P'$, then
$\langle \widehat{\Psi}(P), \widehat{\Psi}(P') \rangle = \frac1{n^2}
\sum_{i,j} \phi(x_i, x_j')$ by the reproducing property. Therefore the
calculation of all pairwise inner products of kernel mean embeddings requires $O(N^2 n^2)$ kernel
evaluations. On the other hand, if we have sparse representations of the
kernel means, these pairwise inner products can be calculated with only
$O(N^2 k^2)$ kernel evaluations, a substantial computational savings. In
our experiments below, we demonstrate the computational savings of our
approach in KME-based algorithms for the embedding of probability
distributions and class-proportion estimation.

\subsection{Generalized Notion of Kernel}

The problem of sparsely approximating a sample mean
can be addressed more generally in an inner product space. This motivates
the following definition of kernel, which is satisfied by both density
estimation kernels and symmetric positive definite kernels.
\begin{defn}
\label{def:kernel}
We say that $\kernel$ is a {\em kernel} if there exists an inner product
space $\sH$ such that for all $x$ in $\bbR^d$, $\phi(\cdot,x) \,\, \in \sH$.
\end{defn}
In the case of kernel density estimation, all commonly used kernels satisfy $\phi(\cdot, x) \in L^2(\rd)$ for all $x \in \rd$. Recalling that $L^2(\rd)$ consists of equivalence classes of functions, when we write $\phi(\cdot, x) \in L^2(\rd)$, we view $\phi(\cdot,x)$ as a representative of its equivalence class. In the case of the kernel mean embedding, we may simply take $\sH$ to be the RKHS associated with $\phi$.

Our proposed methodology applies to kernels of a particular form, given by
the following definition.
\begin{defn} \label{def:radial}
We say $\kernel$ is a {\em radial kernel} if $\phi$ is a kernel as in Def. \ref{def:kernel} and there exists a strictly decreasing function $g: [0,\infty) \to \bbR$ such that, for all $x, x' \in \rd$,
$$
\langle \phi(\cdot, x), \phi(\cdot, x') \rangle_{\sH} = g(\| x - x' \|_2
).
$$
\end{defn}

We now review some common examples of radial kernels.
The Gaussian kernel with parameter $\sigma > 0$ has the form
$$
\phi(x,x') = c_\sigma \exp{\paren{-\frac{\norm{x-x'}_2^2}{2\sigma^2}}},
$$
the Laplacian kernel with parameter $\gamma >0$ has the form
$$
\phi(x,x') =c_\gamma \exp{\paren{-\frac{\norm{x-x'}_2}{\gamma}}},
$$
and the Student-type kernel with parameters $\alpha, \beta >0$ has the
form
$$
\phi(x,x')=c_{\alpha,\beta}\paren{1+\frac{\norm{x-x'}_2^2}{\beta}}^{-\alpha}.
$$
The parameters $c_\sigma, c_\gamma$ and $c_{\alpha,\beta}$ can be set to 1
for the KME, or so as to normalize $\phi$ to be a density estimation
kernel, depending on the application.

These examples illustrate that the space $\sH$ such that $\phi(\cdot, x) \in
\sH$ is not unique. Indeed, each of these three kernels is a
symmetric positive definite kernel, and therefore we may take $\sH$ to
be the RKHS associated with $\phi$ \cite{steinwart08, scovel2010radial}. On the other hand, we may also select
$\sH = L^2(\rd)$.

Each of these three examples is also a radial kernel. If we take $\sH$ to
be the RKHS, then by the reproducing property we simply have $\langle
\phi(\cdot, x), \phi(\cdot, x') \rangle = \phi(x,x')$, and in each case,
$\phi(x,x') = g(\|x - x'\|)$ for some strictly decreasing $g$. These
kernels are also radial if we take $\sH = L^2(\rd)$. For example, consider
the Gaussian kernel, and let us write $\phi = \phi_\sigma$ to indicate the
dependence on the bandwidth parameter. Then $\langle \phi_\sigma(\cdot,
x), \phi_\sigma(\cdot, x') \rangle_{L^2} = \phi_{\sqrt{2} \sigma}(x,x')$.
Similarly, for the Student kernel with $\alpha = (1+d)/2$ (the Cauchy
kernel), we have $\langle \phi_\beta(\cdot, x), \phi_\beta(\cdot, x')
\rangle_{L^2} = \phi_{2 \beta}(x,x')$. For other kernels, although there may not
be a closed form expression for $g$, it can still be argued that such a
$g$ exists, which is all we will need.

\section{Abstract Problem Formulation}
\label{sec:abstractformulation}
In the interest of generality and clarity, we consider the problem of
sparsely approximating a sample mean in a more abstract setting. Thus, let
$\paren{\sH, \inpr{\cdot,\cdot}}$ be an inner product space with induced
norm $\norm{\cdot}_\sH$, and let $\set{z_1,\dots,z_n} \subset \sH$. For $\alpha
\in \bbR^n$, define $\norm{\alpha}_0:=\abs{\set{i \mid \alpha_i \neq 0}}$.
Given an integer $k\leq n$, our objective is to approximate the sample
mean $\zbar = \frac1{n} \sum_i z_i$ as a $k$-sparse linear combination of
$z_1,\ldots,z_n.$ In particular, we want to solve the problem
\begin{align} \label{eq:minproblem}
\text{minimize } & \norm{\zbar-z_\alpha}_\sH \\
\text{subject to } & \norm{\alpha}_0=k \nonumber
\end{align}
where $z_\alpha=\sum_{i\in\brac{n}}\alpha_i z_i$.

Note that problem \eqref{eq:minproblem} is of the form of the standard
sparse approximation problem \cite{tropp2004greed}, where $\{z_1, \ldots,
z_n\}$ is the so-called {\em dictionary} out of which the sparse
approximation is built. Later we argue that existing sparse approximation
algorithms are not suitable from a scalability perspective.  Instead, we
develop an approach that leverages the fact that the vector being sparsely
approximated is the sample mean of the dictionary elements. We are most interested in the case where $z_i = \phi(\cdot,x_i)$ and $\phi$ is a kernel, but the discussion in Section \ref{sec:subsetselection} is held in a more abstract sense.

\section{Related Work and Contributions} \label{sec:relatedwork}
Problem \eqref{eq:minproblem} is a specific case of the sparse approximation problem. Since in general it is NP-hard many efforts have been made to approximate its solution in a feasible amount of time. See \cite{tropp2004greed} for an overview. A standard method of approximation is Matching Pursuit. Matching Pursuit is a greedy algorithm originally designed for finite-dimensional signals. Following the notation of Problem \eqref{eq:minproblem} let $\zbar$ be the target vector we wish to approximate. In Matching Pursuit the first step is to pick an ``atom" in $\set{z_1,\dots,z_n}$ which captures most of $\zbar$ as measured by the magnitude of the inner product. After this first step the subsequent atoms are iteratively chosen according to which one captures more of the portion of $\zbar$ that hasn't been accounted for \cite{mallat1993matching}. Note that just the first step of this algorithm requires to compute, for each $z_i$, the quantity $\inpr{\zbar,z_i} = \frac{1}{n} \sum_{j \in [n]} \inpr{z_i,z_j}$. Since we have $n$ $z_i$'s, the first step already takes $\Omega(n^2)$ kernel evaluations, which is undesirable. Another common approach, Basis Pursuit, has similar time complexity.

Several algorithms which focus specifically on the sparse KDE case have been developed. In \cite{jeon1994fast} a clustering method is used to approximate the KDE at a point by rejecting points which fail  to belong to close clusters. In \cite{girolami2003probability} a relevant subset of the data is chosen to minimize the $L^2$ error but at an expensive $O(n^2)$ cost. In \cite{chen2008orthogonal, schaffoner2007memory} a regression based approach is taken to estimate the KDE through its cumulative density function. Notice these algorithms rely heavily on the assumption that the KDE represents a probability distribution, so cannot be generalized to other kernel means.

When the kernel mean is thought of as a mixture model, the model can be collapsed into a simpler one by reducing the number of its components through a similarity based merging procedure \cite{scott2001kernels, runnalls2007kullback, schieferdecker2009gaussian}. Since these methods necessitate the computation of all pairwise similarities, they present quadratic computational complexity. EM algorithms for this task result in similar computational requirements \cite{figueiredo2002unsupervised, bruneau2010parsimonious}.

A line of work which tries to speed up general kernel sums comes historically from $n$-body problems in physics, and makes use of fast multipole methods \cite{greengard1987fast, gray2000nbody}. The general idea behind these methods is to represent the kernel in question by a truncated series expansion, and then use a space partitioning scheme to group points, yielding an efficient way to approximate group-group or group-point interactions, effectively reducing the number of kernel evaluations. These methods are usually kernel-dependent and do not yield a valid density. For the case of the Gaussian kernel, see \cite{yang2003improved,lee2006dualfgt} for two different space partitioning methods. Note that, contrary to these methods, our approach can still yield a valid density (discussed below), and can therefore be used to estimate quantities like the KL divergence.

The efforts of rapidly approximating general kernel based quantities have led to the use of $\epsilon$-samples, or coresets. To define $\epsilon$-samples, first denote the data $A := \data$ and the kernel quantity of interest $Q(A,x)$, where $x$ is some query point (for example, the KDE is $Q(A,x) = \frac{1}{n} \sum_{i\in [n]} \phi(x_i,x)$). An $\epsilon$-sample is a set $A' \subset A$ such that, for every query point $x$, $Q(A,x)$ and $Q(A',x)$ differ by less than $\epsilon$ with respect to some norm. See \cite{zheng2013quality, phillips2013epsilon} for the KDE case with $\ell_\infty$ norm. For other kernel quantities, in specific the KME using the RKHS norm, see \cite{Joshi2011comparing}. Both cases allow for constructions of $\epsilon$-samples in near linear time with respect to the data size and $1/\epsilon$. Notice that our approach has the advantage that it handles both the KDE and KME cases simultaneously, and that if desired it can yield a valid density as the approximation.

Although most of the literature seems to concentrate on the KDE, there have also been efforts to speed up computation time in problems involving the KME. As in the $\epsilon$-sample approach above, many of these problems require the distance between KMEs in the RKHS, so they focus on speeding up this calculation. In \cite{zhao2014fastmmd}, for example, a fast method is devised for the specific case of the maximum mean discrepancy statistic used for the two-sample test.

Computing the kernel mean at each of the original points $\data$ can be thought of as a matrix vector multiplication, where the matrix in question is the kernel matrix. Therefore, an algebraic approach to this problem consists of choosing a suitable subset of the matrix columns and then approximating the complete matrix only through these columns. Among the most common of these is the Nystr\"om method. In the Nystr\"om method the kernel matrix $K$ is approximated by the matrix $QW_r^+Q^T$, where $Q$ is composed of a subset of the columns of $K$, W is those columns intersected with their corresponding rows, and $W_r^+$ the best $r$-rank approximation to its pseudoinverse (see \cite{drineas2005nystrom} for details). The columns composing $Q$ are typically chosen randomly under some sampling distribution. See \cite{kumar2012sampling} for some examples of sampling distributions. As explained in Section \ref{sec:nystrom}, our approach is connected to the Nystr\"om method and can be viewed as a particular scheme for column selection tailored to kernel means. The Nystr\"om approximation of the kernel matrix is not the only one used though, and other algebraic approaches exist. In \cite{march2014far} for example, an interpolative decomposition of the kernel matrix is proposed.

In \cite{noumir2012one} a ``coherence" based sparsification criterion is used in the context of one-class classification. The main idea is that each set of possible atoms $\set{z_i | \alpha_i\neq 0}$ can be quantified by the largest absolute value of the inner product between two different atoms. The method proposed requires the computation of the complete kernel matrix, and is therefore not suitable for our setting, which involves large data. The motivation for their coherence criterion, however, lies in the minimization of a bound on the approximation error. As seen in Section \ref{sec:bound}, we propose a similar bound as a starting point for our algorithm.

\subsection*{Contributions} \label{sec:contributions}
We list a summary of contributions in this paper.
\begin{itemize}
\item We present a bound on the sparse approximation error based on a novel measure of incoherence.
\item We recognize that for radial kernels, minimizing the bound is equivalent to solving an instance of the $k$-center problem. The solution to the $k$-center problem, in turn, can be approximated by a linear running time algorithm.
\item Our method for approximating the KDE can be implemented so that the sparse kernel mean is a valid density function, which is important for some applications.
\item Our method provides amortization of computational complexity since the calculation of the set $\sI$ (introduced below) is only computed once. Many subsequent calculations (e.g., kernel bandwidth search) can then be performed at a relatively small or negligible cost.
\item Our method is flexible in that it addresses different types of kernel means. In particular, it can be used to approximate both KMEs and KDEs.
\item Our method provides a scheme to automatically select the sparsity level.
\item We demonstrate the improved performance of the proposed method in three different applications: Euclidean embedding of probabilities (using both the KDE and the KME), class proportion estimation (using the KME), and clustering with the mean-shift algorithm (using the KDE).
\end{itemize}

\section{Subset Selection and Incoherence-Based Bound} \label{sec:subsetselection}
Let us now reformulate problem \eqref{eq:minproblem}. Our approach will be to separate the problem into two parts: that of finding the set of indices $i$ such that $\alpha_i$ is not zero, and that of finding the value of the nonzero $\alpha_i$'s. Letting $\sI \subset \brac{n}$ denote an index set, we can pose problem \eqref{eq:minproblem} as
\begin{equation} \label{eq:minalpha}
\underset{|\sI|=k}{\underset{\sI \subseteq [n]} {\min}}   \, \, \underset{(\alpha_i)_{i \in \sI}}\min   \| \zbar-\underset{i \in \sI}\sum \alpha_i z_i\|^2 \, .
\end{equation}
Note that the inner optimization problem is unconstrained and quadratic, and its solution, which for fixed $\sI$ and $k$ we denote by $\alpha_\sI \in \bbR^k$, is
$$
\alpha_\sI = K_\sI^{-1}\kappa_\sI,
$$
where $K_\sI = \paren{\inpr{z_i,z_j}}_{i,j \in \sI}$ and $\kappa_\sI$ is the $k$-dimensional vector with entries $\frac{1}{n}\sum_{j\in\brac{n}} \inpr{z_j,z_l}$, $\, l \in \sI$.

Let $\alpha_\sI=\paren{\alpha_{\sI,i}}_{i\in\sI}$ and $z_\sI=\sum_{i\in\sI}\alpha_{\sI,i}z_i$. Then we can rewrite problem \eqref{eq:minproblem} as
\begin{equation} \label{eq:mini}
\underset{\abs{\sI}=k}{\underset{\sI \subseteq [n]} \min} \norm{\zbar - z_\sI}.
\end{equation}

\subsection{Connection to the Nystr\"om Method} \label{sec:nystrom}
Before continuing to the approximate solution of problem \eqref{eq:mini}, we briefly highlight its relationship to the Nystr\"om method. Given a set $\sI \subset [n]$, let $K$ be the kernel matrix of $\set{z_i | i \in [n]}$, $K:=(\inpr{z_i,z_j})_{i,j\in [n]}$, and $K_\sI$ the kernel matrix of $\set{z_i | i \in \sI}$, $K_\sI:=(\inpr{z_i,z_j})_{i,j\in \sI}$. Also, let $Q_\sI$ be the binary matrix such that $KQ_\sI$ is composed of the columns of $K$ corresponding to $\sI$. Then we can rewrite $\alpha_\sI$ and $K_\sI$ as $\alpha_\sI = \paren{Q_\sI^TKQ_\sI}^{-1}Q_\sI^TK\uno_n$ and $K_\sI=Q_\sI^TKQ_\sI$, where $\uno_n$ denotes the vector in $\bbR^n$ with entries $1/n$. By doing so, we can express the objective of \eqref{eq:mini} as
\begin{align}
\norm{\zbar - z_\sI}^2  & = \uno_n^T\paren{K-KQ_\sI K_\sI^{-1}Q_\sI^TK^T}\uno_n \nonumber \\
& = \uno_n^T\paren{K-\tilde{K}_\sI}\uno_n . \nonumber
\end{align}
where $\tilde{K}_\sI:=KQ_\sI K_\sI^{-1}Q_\sI^TK^T$. We recognize $\tilde{K}_\sI$ as the Nystr\"om matrix from the Nystr\"om method \cite{kumar2012sampling}, which is the only term dependent on $\sI$ in the objective. Therefore, our work can be interpreted from the Nystr\"om perspective: choose suitable columns of $K$ and approximate $K$ through the Nystr\"om matrix. The main difference is that the resulting approximation is based on the induced norm of the inner product space where the $z_i$'s reside, instead of the commonly used spectral and Frobenius norms.

\subsection{An Incoherence-based Sparse Approximation Bound}
\label{sec:bound}
We now present our proposed algorithm to approximate the solution of problem \eqref{eq:mini}. Our strategy is to find an upper bound on the term $\norm{\zbar - z_\sI}$ which is dependent on $\sI$ and then find the $\sI$ that minimizes the bound. First, we present a lemma which will aid us in finding the bound.

\begin{lemma} \label{lemma}
Let $\paren{\sH,\inpr{\cdot,\cdot}}$ be an inner product space. Let $S$ be a finite dimensional subspace of $\sH$ and $P_S$ the projection onto $S$. For any $z_0 \in \sH$
$$
\norm{P_Sz_0} = \max_{z \in S, \norm{z}=1} \inpr{z_0,z}.
$$
\end{lemma}
\begin{proof}
First note that since $S$ is finite dimensional, by the Projection Theorem $z_0 - P_Sz_0$ is orthogonal to $S$. Now, for any $z\in S$ with $\norm{z}=1$, we have
\begin{align*}
\inpr{z_0,z} &= \inpr{P_Sz_0 + (z_0-P_Sz_0),z} \\
&= \inpr{P_Sz_0,z} + \inpr{z_0-P_Sz_0,z} \\
&= \inpr{P_Sz_0,z} \\
&\leq \norm{P_Sz_0}\norm{z} = \norm{P_Sz_0},
\end{align*}
where we have used the Cauchy-Schwartz inequality. To confirm the existence of a vector $z$ which makes it an equality and therefore reaches the maximum, just let $z = P_Sz_0 / \norm{P_Sz_0}$.
\end{proof}

We can now present the theorem which will be the basis for our minimization approach. First, define
$$
\nu_{\sI}:=\underset{j\notin\sI}{\min}\,\,\underset{i\in\sI}{\max}\,\,\inpr{z_i,z_j},
$$
which we can think of as a measure of the ``incoherence" of $\set{z_i \mid i\in\sI}$. It is now possible to establish a bound:
\begin{thm}\label{bound}
Assume that for some $C>0$ $\inpr{z_i,z_i}=C \,\, \forall i \in \brac{n}$. Then for every $\sI\subseteq\brac{n}$,
$$
\norm{\zbar - z_\sI} \leq \paren{1-\frac{\abs{\sI}}{n}}\sqrt{\frac{1}{C}\paren{C^2 - \nu_\sI^2}} .
$$
\end{thm}

\begin{proof}
The beginning of this proof is similar to the one in \cite{noumir2012one}.
Let $S_\sI :=\text{span}(\{ z_i \, | \, i \in \sI \})$ and denote $P_{S_\sI}$ the projection operator onto $S_\sI$ and $I$ the identity operator. We have
\begin{align*}
\| \zbar-z_\sI\| =\| \zbar-P_{S_\sI} \zbar\| = \frac{1}{n}\| \sum_{i \in [n]} (I-P_{S_\sI}) z_i\|
\\ \leq\frac{1}{n}\sum_{i \in [n]}\| (I-P_{S_\sI})z_{i}\| =\frac{1}{n}\underset{i\notin\sI}{\sum}\|(I-P_{S_\sI})z_i\|
\end{align*}
where we have used the triangle inequality, and the last equality is due to the fact that $z_i=P_{S_\sI} z_i$ when $z_i\in S_\sI$.

Now, since $(z_i-P_{S_\sI} z_i) \perp P_{S_\sI} z_i$, we can use Pythagoras' Theorem in $\sH$ to get $\|z_i - P_{S_\sI} z_i\|^2 = \|z_i\|^2 - \|P_{S_\sI} z_i\|^2$.

By Lemma \ref{lemma}, $\| P_{S_\sI} z_{i}\| =\underset{z\in S_\sI,\,\| z\| =1}{\max}\left< z_{i},z\right> $. Therefore, for $i \notin \sI$,
\begin{align*}
\| P_{S_\sI} z_i \| &= \frac1{\sqrt{C}} \, \underset{z \in S_\sI, \| z\| = \sqrt{C}}{\max} \, \langle z_i, z \rangle   \\
&\ge \frac1{\sqrt{C}} \, \underset{\ell \in \sI}{\max} \, \langle z_i, z_\ell \rangle \\
&\ge \frac1{\sqrt{C}} \, \underset{j \notin \sI}{\min} \, \underset{\ell \in \sI}{\max} \, \langle z_j, z_\ell \rangle = \frac1{\sqrt{C}} \, \nu_\sI.
\end{align*}
Thus, for $i \notin \sI$,
$$
\| z_i \|^2 - \| P_{S_\sI} z_i \|^2 \le C - \frac{\nu_\sI^2}{C}
$$
and finally
$$
\| \bar{z}-z_\sI \| \leq \frac{1}{n} \underset{i \notin \sI}{\sum}\sqrt{C-\frac{\nu_\sI^{2}}{C}}
=\left(1-\frac{|\sI |}{n}\right)\sqrt{\frac{1}{C}(C^{2}-\nu_\sI^2)} \, .
$$
\end{proof}

\section{Bound Minimization Via $k$-center Algorithm} \label{sec:skm}
In this section we apply the previous result in the context of approximating a kernel mean based on a radial kernel. Recall that, in the kernel mean setting, $z_i = \phi(\cdot,x_i)$ and $\inpr{\phi(\cdot,x_i),\phi(\cdot,x_j)} = g(\norm{x_i-x_j}_2)$, where $\phi$ is a radial kernel, $\data \subset \rd$, and $g$ is strictly decreasing as in Definition \ref{def:radial}. Also note that for any radial kernel the assumption in Theorem \ref{bound} is satisfied, since $\inpr{\phi(\cdot,x_i),\phi(\cdot,x_i)} = g(0)=C>0$.

Define the set $\sI^*$ as
$$
\sI^*:=\arg\,\, \underset{\abs{\sI}=k} { \underset{\sI \subseteq [n]} {\min} }\,\,\underset{j\notin\sI}{\max}\,\,\underset{i\in\sI}{\min}\,\,\norm{ x_{i}-x_{j}}.
$$
Then, since $\phi$ is a radial kernel and $g$ is strictly decreasing, $\sI^*$ also maximizes $\nu_\sI = \underset{j\notin\sI}{\min}\,\,\underset{i\in\sI}{\max}\,\, g(\norm{x_i-x_j})$. Therefore, $\sI^*$ is the set that minimizes the bound in Theorem \ref{bound}. We have translated a problem involving inner products of functions to a problem involving distances between points in $\rd$.

The problem of finding $\sI^*$ is known as the $k$-center problem. To pose the $k$-center problem more precisely, we make a few definitions. For a fixed $\sI$, let $X_\sI=\set{x_i \mid i\in\sI}$ and $Y_\sI=\set{x_j \mid j\notin \sI}$, and for all $x_j\in Y_\sI$ define its distance to $X_\sI$ as $d(x_j,X_\sI)=\underset{x_i\in X_\sI}\min\norm{x_i-x_j}$. Furthermore, let $W(X_\sI)=\underset{x_j\in Y_\sI}\max d(x_j,X_\sI)$. Therefore, the $k$-center problem is that of finding the set $\sI$ of size $k$ for which $W(X_\sI)$ is minimized.

The $k$-center problem is known to be NP-complete \cite{vazirani2001approximation}. However, there exists a greedy 2-approximation algorithm \cite{gonzalez1985clustering} which produces a set $\sI_k$ such that $W(X_{\sI_k})\leq2W(X_{\sI^*})$. This algorithm is optimal in the sense that under the assumption that P$\neq$NP there is no $\rho$-approximation algorithm with $\rho<2$ \cite{hochbaum1996approximation}. The algorithm is described in Fig. \ref{alg:kalg}, and as can be seen, it has a linear time complexity in the size of the data $n$. In particular, the algorithm runs in $O(nkd)$ time.

\begin{figure}[htp]
\begin{algorithmic}
\STATE {\bf input} $x_1, \ldots, x_n, k$
\STATE $X \longleftarrow \varnothing$
\STATE $Y \longleftarrow \{ x_1, \ldots, x_n\}$
\STATE Choose randomly a first index $u \in [n]$
\STATE $X \longleftarrow X \cup \{ x_u\}$
\STATE $Y \longleftarrow Y \backslash \{ x_u\}$
\WHILE{$|X|<k$}
\STATE	Choose the element $y\in Y$ for which $d(y,X)$ is maximized
\STATE	$X \longleftarrow X \cup \{ y\}$
\STATE	$Y \longleftarrow Y \backslash \{ y\}$
\ENDWHILE
\STATE {\bf output} {$\sI_k = \{ i \in [n] \, | \, x_i \in X \}$}
\end{algorithmic}
\caption{A linear time $2$-approximation algorithm for the $k$-center problem.}
\label{alg:kalg}
\end{figure}

\subsection{Computation of $\alpha_\sI$ and Auto-selection of $k$}
\label{sec:autoselect}
The $k$-center algorithm allows us to find the set $\sI$ on which our approximation will be based. After finding $\sI$ we can determine the optimal coefficients $\alpha_\sI$. Since the main computational burden is in the selection of $\sI$, we now have the freedom to explore different values of $\alpha_\sI$ in a relatively small amount of time. For example, we can compute $\alpha_\sI$ for each of several possible kernel bandwidths $\sigma$.

The optimal way to compute $\alpha_\sI$ depends on the application. If the user has a good idea of what the value of $k$ is, then a fast way to compute $\alpha_\sI$ for that specific value is to apply their preferred method to solve the equation $K_\sI \alpha_\sI = \kappa_\sI$. For example, since for symmetric positive definite kernels the kernel matrix is positive semi-definite, the preconditioned conjugate gradient method can be used to quickly obtain $\alpha_\sI$ to high accuracy. This approach has the advantages of being simple and fast.

A further advantage of our method is evident when the user has access only to a maximum tolerance value of $k$, say $k_{max}$, but desires to stop at a value $k_0\leq k_{max}$ which performs as well as $k_{max}$. To do this, at iteration $m\geq 1$ in the $k$-center algorithm we compute $\alpha_{\sI_m}$ right after computing $\sI_m$, which provides a record of all the $\alpha_{\sI_j}$ for $1\leq j \leq k_0$. To find $k_0$, we use the information from the computed coefficients to form an error indicator and stop when some error threshold is overcome. Before showing what these error indicators are, we first provide an update rule to efficiently compute the $\alpha$ coefficients at each iteration step.

Let $\sI_{m}$ be the set of the first $m$ elements chosen by the $k$-center algorithm, and let $\alpha_{\sI_m}$, $K_{\sI_m}$ and $\kappa_{\sI_m}$ be obtained by using $\sI_{m}$. If we increase the number of components to $m+1$, then as shown in \cite{noumir2012one} we have
$$
K_{\sI_{m+1}} = \begin{bmatrix}
K_{\sI_m} & b \\
b^T & \phi(x_{j_{m+1}},x_{j_{m+1}})
\end{bmatrix}
$$
where $x_{j_{\ell}}$ is the $\ell^{th}$ element selected by the $k$-center algorithm, and $b=(\phi(x_{j_{m+1}},x_i))_{i\in  \sI_m}$. The resulting update rule for the inverse is
$$
K_{\sI_{m+1}}^{-1} = \begin{bmatrix}
K_{\sI_m}^{-1} & 0 \\
0 & 0
\end{bmatrix}
+ q_0 (qq^t)
$$
where $q_0 = 1/(\phi(x_{j_{m+1}},x_{j_{m+1}}) -b^TK_{\sI_m}^{-1}b)$ and $q = \brac{-b^TK_{\sI_m}^{-1} \,\,\,\,\,\, 1}^T$. From here the user can now compute $\alpha_{\sI_{m+1}}$ by multiplying $K_{\sI_{m+1}}^{-1}$ with
$$
\kappa_{\sI_{m+1}} = \begin{bmatrix}
\kappa_{\sI_m} \\
\frac{1}{n} \sum_{i=1}^{n} \phi(x_{j_{m+1}},x_i)
\end{bmatrix}.
$$
Assuming we stop at some $k_{max}$, the time complexity for computing all the $\alpha_{\sI_m}$'s is $O(k_{max}^3)$ and the necessary memory $O(k_{max}^2)$.

To automatically stop at some $k_0\leq k_{max}$ we need a stopping criterion based on some form of error. We propose the following: using the notation of problem \eqref{eq:mini} we have that
\begin{align*}
&\norm{\zbar {-} z_\sI}^2{=}\:\inpr{\frac{1}{n}\sum_{\ell\in [n]} z_\ell, \frac{1}{n}\sum_{\ell'\in [n]} z_{\ell'} }\nonumber\\
&{-}\:2\inpr{\frac{1}{n}\sum_{\ell\in [n]}z_\ell, \sum_{i\in\sI}\alpha_{\sI,i}z_i}{+}\:\inpr{\sum_{i\in\sI}\alpha_{\sI,i}z_i,\sum_{j\in\sI}\alpha_{\sI,j}z_j} \nonumber \\
&{=}\:\norm{\zbar}^2 - 2 \cdot \sum_{i\in\sI}\alpha_{\sI,i} \cdot \frac{1}{n}\sum_{\ell\in[n]}\inpr{z_\ell,z_i} + \alpha_\sI^T K_\sI \alpha_\sI \nonumber \\
&{=}\:\norm{\zbar}^2 - \alpha_\sI^T\kappa_\sI.
\end{align*}
Since $\norm{\zbar}^2$ is a constant independent of $\sI$, we can avoid its $O(n^2)$ computation and only use the quantities $E_{\abs{\sI}}:=-\alpha_\sI^T\kappa_\sI$ as error indicators. Note that $E_{t}$ is nonincreasing with respect to $t$. Based on this we choose $k_0$ to be the first value at which some relative error is small. In this paper we used the test
$$
\frac{\abs{E_{k_0-1} - E_{k_0}}}{\abs{E_1-E_{k_0}}} \leq \epsilon
$$
for some small $\epsilon$. The overall complexity amounts to $O(nk_0d + k_0^3d)$.

A further consideration for computing $\alpha_\sI$ should be made if the result is desired to be a probability mass function. In this case a $k$-dimensional $\alpha_\sI$ can be projected into the simplex $\Delta^{k-1}:=\set{\nu\in\bbR^k|\sum_{i=1}^{k}\nu_i=1, \nu_i\geq 0 \,\,\,\forall\,\,\, 1\leq i\leq k}$ after being obtained by any of the discussed methods (see \cite{duchi2008efficient}). Alternatively, a quadratic program which takes into account the constraints of non-negativity and $\sum_{i=1}^{k}\alpha_{\sI,i} = 1$ can be solved.

A Matlab implementation of the complete Sparse Kernel Mean procedure can be found at \cite{scott0000website}.

\section{Experiments: Speeding Up Existing Kernel Mean Methods} \label{sec:experiments}
We have implemented our approach in three specific machine learning tasks that require the computation and evaluation of a mean of kernels. In the first of these, we apply our algorithm to the task of dimensionality reduction. In the second, we use it in the setting of class proportion estimation. Finally, we explore its performance when used as part of the mean shift algorithm.

In the following we refer to our algorithm or to the resulting kernel mean as SKM (for Sparse Kernel Mean). We now provide a detailed description of each task and relevant results. The implementation has been done in Matlab.

\subsection{Euclidean Embedding of Distributions} \label{sec:fcyto}
In this experiment we embed probability distributions in a lower dimensional space for the purpose of visualization. Given a collection of $N$ distributions $\set{P_1,\dots,P_N}$, the procedure consists of creating a similarity matrix for some notion of similarity among these distributions and then performing a dimensionality reduction method. We consider two cases. In the first case the similarity matrix will be the distance between the kernel mean embeddings of the distributions in the RKHS (KME case), while in the second case it will be the (symmetrized) KL divergence between KDEs (KDE case). For dimensionality reduction we will use ISOMAP \cite{isomap}. In the setup we have access to each of $N$ distributions $\set{P_1,\dots,P_N}$ through samples drawn from those distributions. The sample drawn from the $\ell^{th}$ distribution is denoted $\set{x_i^{(\ell)}}_{i=1}^{n_\ell}$.

Notice that in the KDE case, in order to compute the KL divergence it is necessary to obtain a valid density function. A particular advantage of our algorithm is that, by choosing the coefficients as described in Section \ref{sec:autoselect}, the resulting sparse approximation is a density function.

Let us start with the KME case, in which the similarity matrix contains the norm of the difference between the distributions' KMEs. The first task is to estimate the KME using some symmetric positive definite kernel $\phi$. For the $\ell^{th}$ distribution, the empirical estimate of its KME is
$$
\widehat{\Psi}(P_\ell) = \frac{1}{n_\ell}\sum_{i=1}^{n_\ell}\phi(\cdot,x_i^{(\ell)}),
$$
with a sparse approximation
$$
\widehat{\Psi}_{0}(P_\ell) = \sum_{i \in \sI^{(\ell)}}\alpha_i^{(\ell)} \phi(\cdot,x_i^{(\ell)}),
$$
for some set $\sI^{(\ell)}$ and $\set{\alpha_i^{(\ell)} | \alpha_i^{(\ell)}\in \bbR}$, where the $\alpha$ coefficients have been computed according to the update method described in Section \ref{sec:autoselect}.

Given all the KMEs, we can now construct a distance matrix. Let $\sH$ be the RKHS of $\phi$. We can use the distance induced by the RKHS to create the matrix $D$, with entries
\begin{align}
D_{\ell,\ell '}:&{=}\:\norm{\widehat{\Psi}(P_\ell) - \widehat{\Psi}(P_{\ell '})}_{\sH}\nonumber\\
&{=}\:\left[\frac{1}{n_{\ell}^2}\sum_{i,j}\phi(x_i^{(\ell)},x_j^{(\ell)}){-}\:2\frac{1}{n_{\ell}n_{\ell '}}\sum_{i,j}\phi(x_i^{(\ell)},x_j^{(\ell ')})\right.\nonumber\\
&{+}\:\left.\frac{1}{n_{\ell '}^2}\sum_{i,j}\phi(x_i^{(\ell ')},x_j^{(\ell ')})\right]^{1/2}.\nonumber
\end{align}
We similarly define $D_0$ based on the sparse KMEs. With such matrix ISOMAP can now be performed to visualize the distributions in, say, $\bbR^2$.

Note that if the samples from $P_\ell$ and $P_{\ell '}$ have $n_\ell$ and $n_{\ell '}$ points, then $D_{\ell , \ell '}$ takes $\Theta(n_\ell^2 + n_\ell n_{\ell '} + n_{\ell'}^2)$ time to compute. Since we need all the pairwise distances, we need $\Theta(N^2)$ such computations. A sparse approximation of the KMEs of $P_\ell$ and $P_{\ell '}$ of sizes $k_\ell$ and $k_{\ell '}$ would instead yield a computation of $\Theta(k_\ell^2 + k_\ell k_{\ell '} + k_{\ell'}^2)$ for each entry. Assuming all samples have the same size $n$, and the sparse approximation size is $k$, then the computation of the distance matrix is reduced from $\Theta(N^2n^2)$ to $\Theta(N^2k^2)$.

Inspired by the work of \cite{finn2009analysis}, we have performed these experiments on flow cytometry data from $N=37$ cancer patients, with sample sizes ranging from 8181 to 108343. We have used the Gaussian kernel, chosen $\sH$ to be its RKHS, and computed the bandwidth based on the `iqr' scale option in R's KernSmooth package. That is, we have computed the interquartile range of the data, averaged over each dimension, and divided by 1.35. After the embedding has been done, we have performed Procrustes analysis on the points so as to account for possible translation and rotation, we also scaled by a suitable factor.

To determine the maximum size $k_\ell$ of each sparse representation, we recall that the SKM procedure takes $O(n_\ell k_\ell + k_\ell^3)$ kernel evaluations, so in order to respect the $n_\ell k_\ell$ factor, we have chosen a small multiple of $\sqrt{n_\ell}$ for $k_\ell$. In this case we picked $k_\ell$ to be the largest integer smaller than $3 \sqrt{n_\ell}$ for each $\ell$. We have implemented the auto-selection scheme described in Section \ref{sec:autoselect}. The results for the case of $\epsilon = 10^{-10}$ are shown in Fig. \ref{fig:fcytoKME} and Table  \ref{tab:fcytoKME}. Although $k_\ell$ is the largest allowed sparsity, the algorithm stops at some $k_{0\ell}\leq k_\ell$. To determine how well $D_0$ approximates $D$, we have plotted the relative error $\frac{\norm{D-D_0}_F}{\norm{D}_F}$ for different values of $\epsilon$, averaged over ten different runs. The result is shown in Fig. \ref{fig:fcytoKMEeps}.

\begin{figure}[htp]
\centering
\includegraphics[width=1\textwidth]{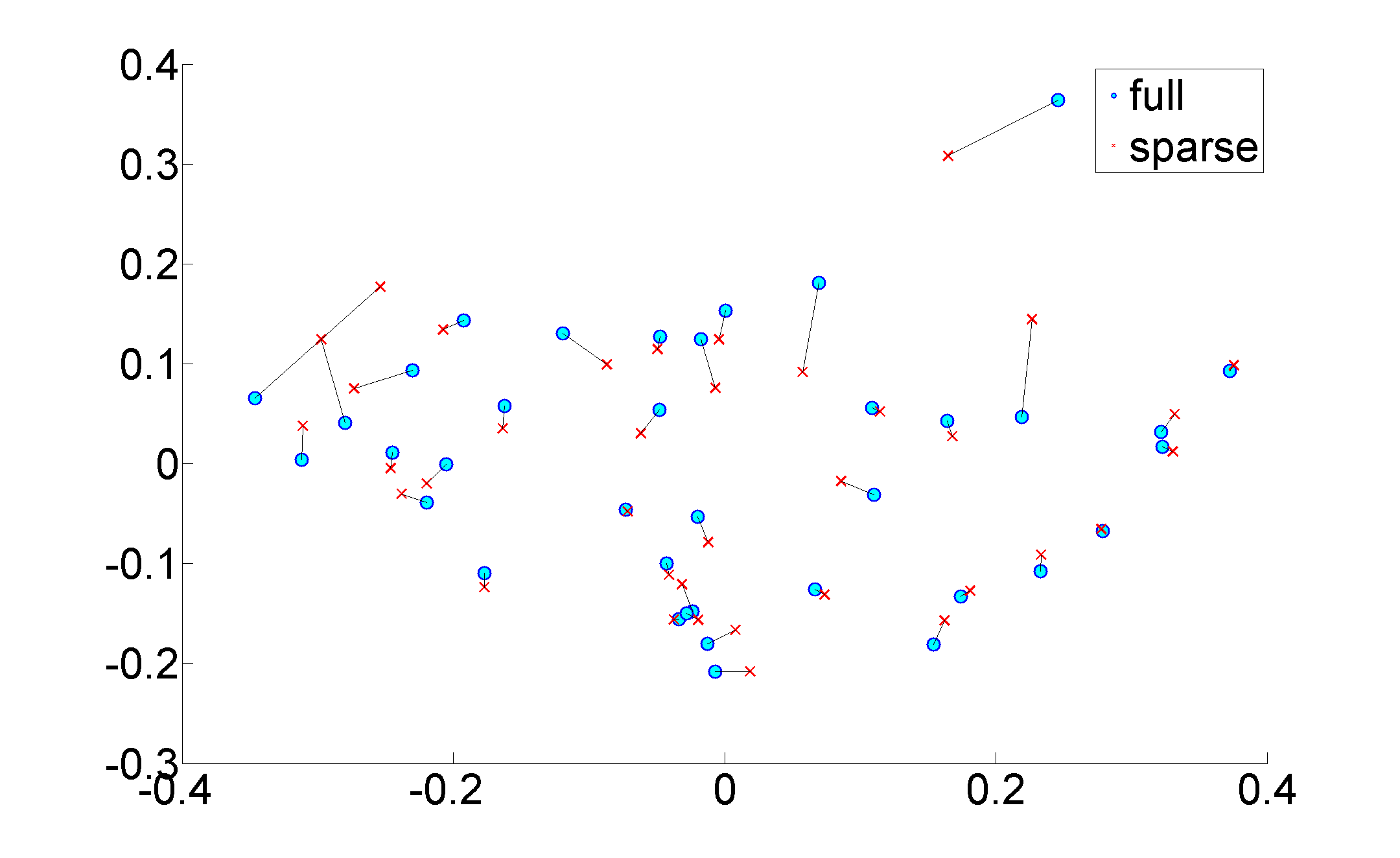}
\caption{2-dimensional representation of flow cytometry data - KME case. Each point represents a patient's distribution. The embeddings were obtained by applying ISOMAP to distances in the RKHS.}
\label{fig:fcytoKME}
\end{figure}

\begin{table}[htp]
  \centering
       \caption{Time comparison for the Euclidean embedding of the flow cytometry dataset - KME case.}
    \begin{tabular}{llll}
    \hline
	& $k$-center  & $D$ computation & Total \\
    \hline
    Full  & 0 & 8.1hrs & 8.1hrs \\
    SKM & 21.7mins & 1.4s & 21.7mins \\
    \hline
    \end{tabular}%
  \label{tab:fcytoKME}%
\end{table}

\begin{figure}[htp]
  \centering
    \includegraphics[width=1\textwidth]{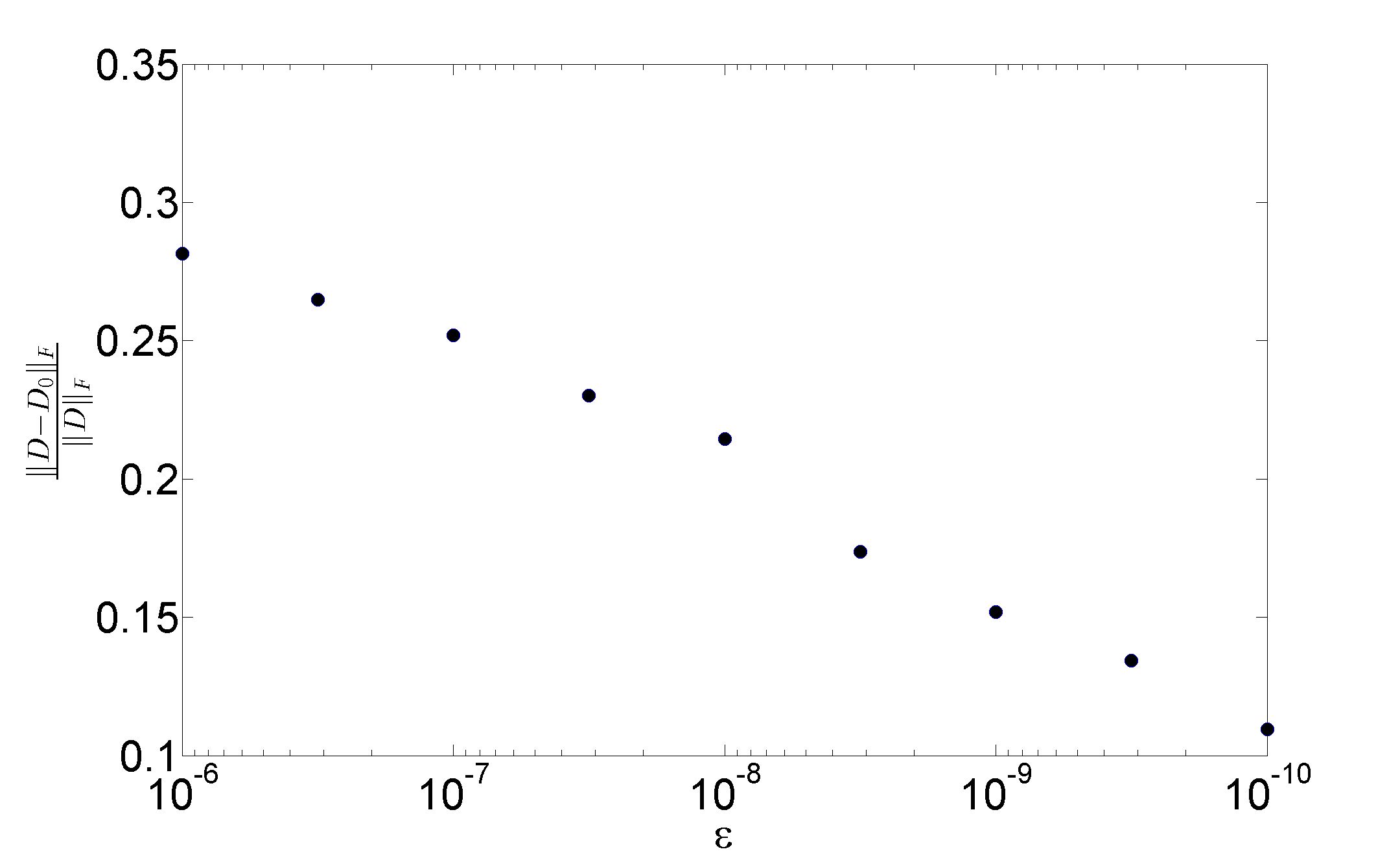}
    \caption{The relative error incurred by the SKM-based matrix $D_0$ as a function of $\epsilon$, averaged over 10 runs - KME case. The average $k$-center and $D_0$ computation times range from 2.4 to 24.6 minutes, and from 0.15 to 2.66 seconds, respectively. The average ratio $k_0/k_{max}$ ranges from 0.13 to 0.81.}
    \label{fig:fcytoKMEeps}
\end{figure}

The KDE case is similar. The similarity matrix is composed of the symmetrized KL divergence between the KDEs of the distributions, defined as $d_{KL}(p,q) := D_{KL}(p \| q) + D_{KL}(q \| p)$, where $D_{KL}$ indicates the KL divergence. For the $\ell^{th}$ distribution, its KDE is
$$
\widehat{f_\ell} = \frac{1}{n_\ell}\sum_{i=1}^{n_\ell}\phi(\cdot,x_i^{(\ell)}),
$$
with a sparse approximation
$$
\widehat{f_0}_{\ell} = \sum_{i \in \sI^{(\ell)}}\alpha_i^{(\ell)} \phi(\cdot,x_i^{(\ell)}).
$$
for some set $\sI^{(\ell)}$ and $\set{\alpha_i^{(\ell)} | \alpha_i^{(\ell)}\geq 0 \,\,\, , \,\,\, \sum_{i}\alpha_i^{(\ell)}=1}$, which has again been calculated according to the update method described in Section \ref{sec:autoselect}. Note that the KL divergence requires two density functions as input, therefore it is important to obtain a valid density. An advantage of our algorithm is that this is possible by obtaining the $\alpha$ coefficients and then projecting into the simplex as indicated in \ref{sec:autoselect}. As in the KME case, we construct the similarity matrix $(D_{\ell,\ell'}):=d_{KL}(\widehat{f_\ell}, \widehat{f_{\ell'}})$.

To compute the KL divergence we split the data in two, use the first half for estimation of the KDE, and the second half for evaluation of the KL divergence. We have chosen $k_\ell = 3\sqrt{n_\ell}$ for each $\ell$, as in the KME case, and used the same stopping criterion with $\epsilon = 10^{-10}$. The results for $\epsilon=10^{-10}$ are shown in Fig. \ref{fig:fcytoKDE} and Table \ref{tab:fcytoKDE}, the plot of $\frac{\norm{D-D_0}_F}{\norm{D}_F}$ for several values of $\epsilon$ is shown in Fig. \ref{fig:fcytoKDEeps}.

Figs. \ref{fig:fcytoKME} and \ref{fig:fcytoKDE} show us that the resulting embedded points using the sparse approximation keep the structure as of those using the full kernel means. Notice also from the Tables that the sparse approximation is many times faster than the full computation (about 20 times faster for each case). Furthermore, in the KME case, the main computational investment is made in finding the elements of the sets $\sI^{(\ell)}$, since the subsequent computation of $D_0$ is of negligible time.

\begin{figure}[htp]
  \centering
    \includegraphics[width=1\textwidth]{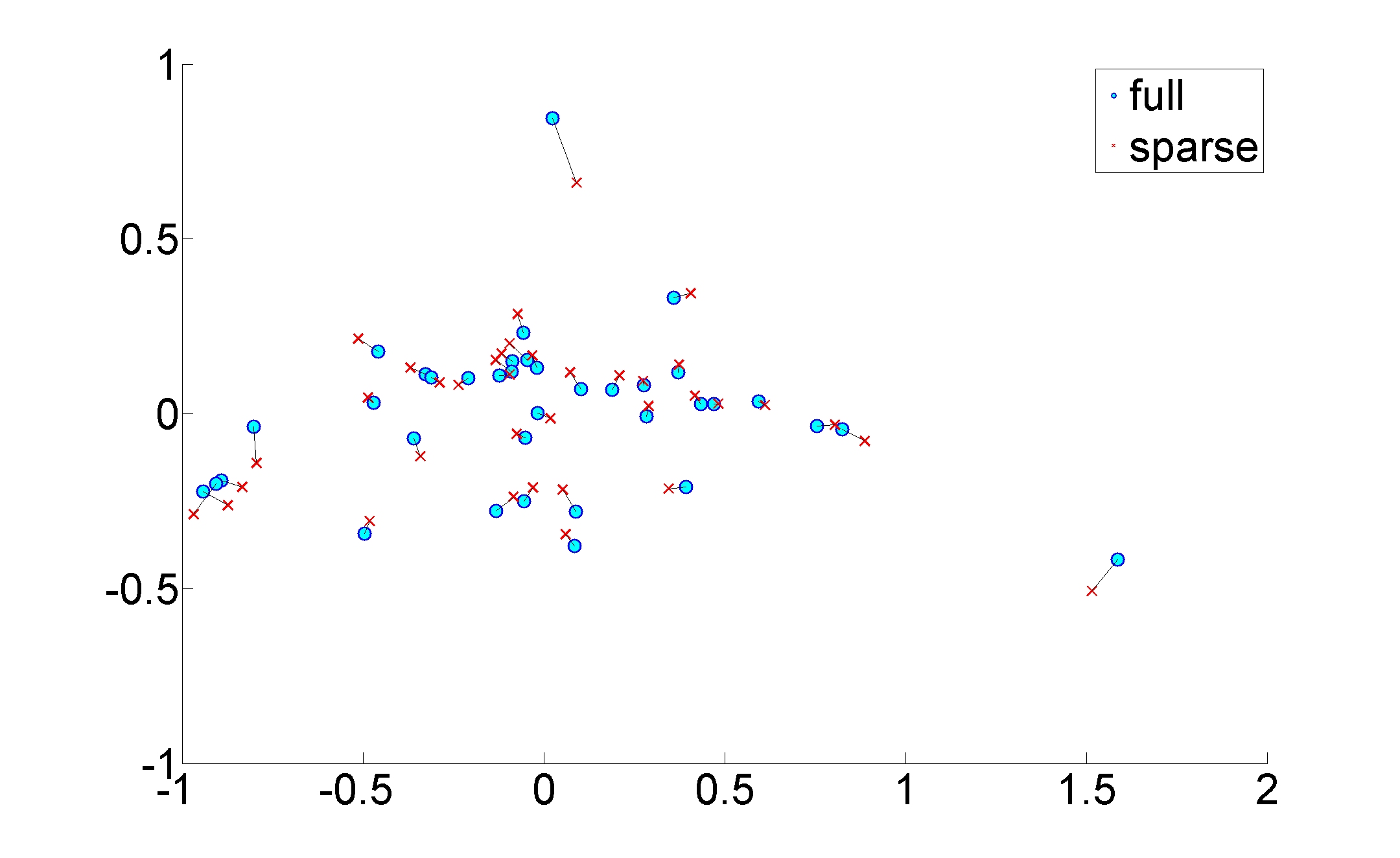}
    \caption{2-dimensional representation of flow cytometry data - KDE case. Each point represents a patient's distribution. The embeddings were obtained by applying ISOMAP to distances between KDEs as measured by the KL divergence.}
\label{fig:fcytoKDE}
\end{figure}

\begin{table}[htp]
  \centering
         \caption{Time Comparison for the Euclidean embedding of the Flow Cytometry dataset - KDE case.}
    \begin{tabular}{llll}
    \hline
	& $k$-center  & $D$ computation & Total \\
    \hline
    Full  & 0 & 2.18 hrs & 2.18 hrs \\
    SKM & 5mins & 2mins & 7mins \\
    \hline
    \end{tabular}%
  \label{tab:fcytoKDE}%
\end{table}

\begin{figure}[htp]
  \centering
    \includegraphics[width=1\textwidth]{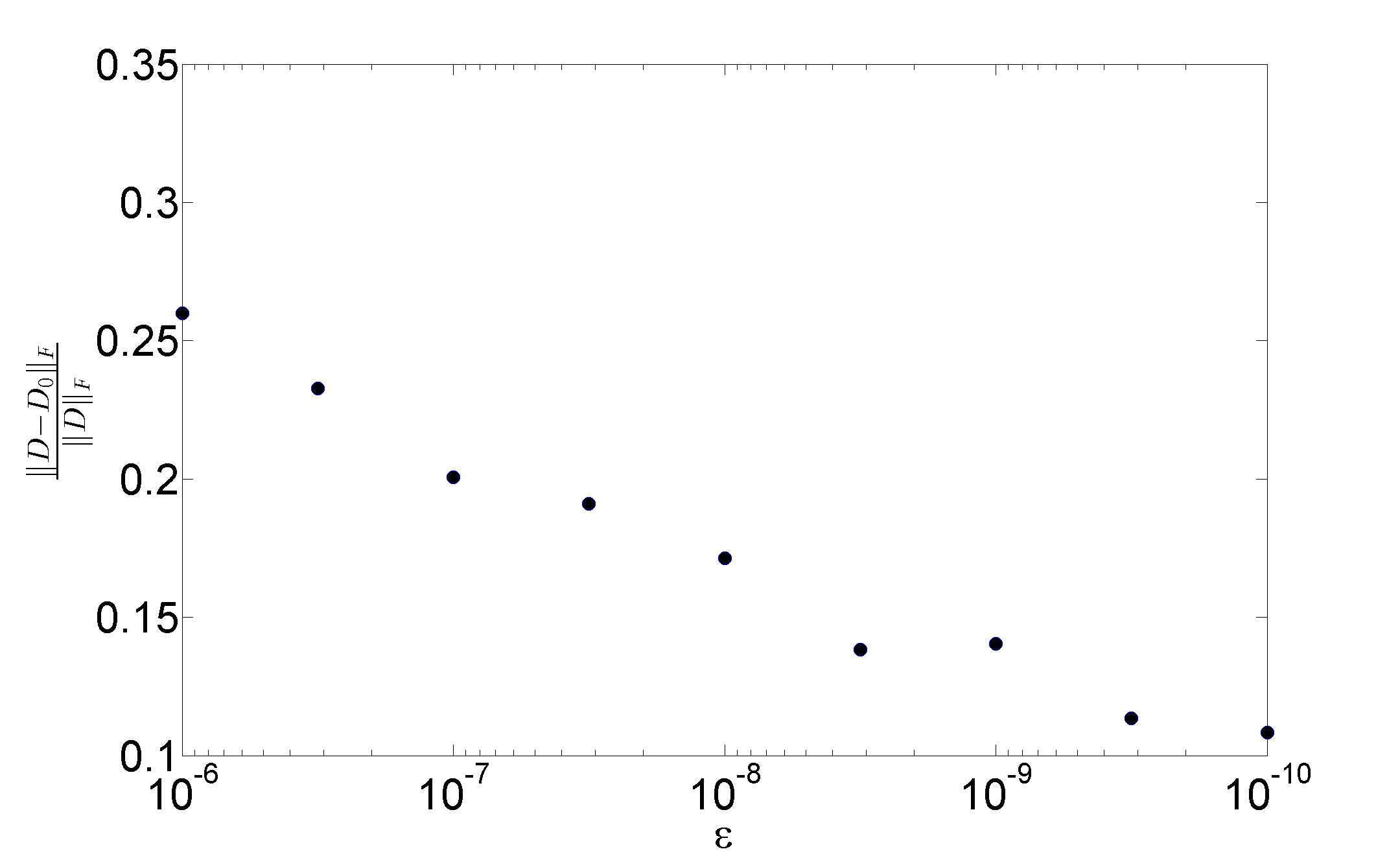}
    \caption{The relative error incurred by the SKM-based matrix $D_0$ as a function of $\epsilon$, averaged over 10 runs - KDE case. The average $k$-center and $D_0$ computation times range from 1 to 7.25 minutes, and from 35 to 160 seconds, respectively. The average ratio $k_0/k_{max}$ ranges from 0.22 to 0.9.}
\label{fig:fcytoKDEeps}
\end{figure}

\subsection{Class Proportion Estimation}
In this setting we are presented with labeled training data drawn from $N$ distributions $\set{P_1,\dots,P_N}$ and with further testing data drawn from a mixture of these distributions $P_0 = \sum_{i=1}^{N}{\pi_i P_i}$, where $\pi_i \geq 0$ and $\sum_i{\pi_i} = 1$. Our goal is to estimate the mixture proportions $\set{\pi_1,\dots,\pi_N}$.

To do so we let $\hat{P}_\ell$ represent the KME of $P_\ell$ for $0\leq \ell \leq N$. We then find the proportions $\set{\hat{\pi}_i}_{i=1}^N$ that minimize the distance
$$
\|{\hat{P}_0 - \sum_{i=1}^N {\pi}_i \hat{P}_i} \| ^2_\sH,
$$
where $\sH$ is the RKHS of the kernel used to construct the KME. By setting the derivative to zero the optimal vector of proportions $\hat{\pi}_{-} := \brac{\hat{\pi}_1,\dots,\hat{\pi}_{N-1}}^T$, subject to $\sum_{i=1}^N{\hat{\pi}_i} = 1$ but not to $\hat{\pi}_i \geq 0$, satisfies
$$
\hat{D}\hat{\pi}_{-}=\hat{e},
$$
where
$$
\hat{D}_{ij} = \inpr{\hat{P}_i - \hat{P}_N, \hat{P}_j - \hat{P}_N}_{\sH}
$$
and
$$
\hat{e}_i = \inpr{\hat{P}_i - \hat{P}_N, \hat{P}_0 - \hat{P}_N}_{\sH}.
$$
From here we can define
$$
\hat{\pi}:= \begin{bmatrix}
\hat{\pi}_{-} \\
1-\sum_{i=1}^{N-1}\hat{\pi}_i
\end{bmatrix}.
$$
A parallel approach, using the KDE instead of the KME is shown in \cite{titterington83minimum}. In that case the distance in $\sH$ was changed to the $L^2$ distance.

Notice we have not enforced the constraint $\hat{\pi}_i \geq 0$, for $1\leq i \leq N$. To do so a quadratic program can be set. For most of our simulations we did not encounter the necessity to do so. Therefore, for the few cases for which $\hat{\pi}$ lied outside of the simplex, we have projected onto it as described in \cite{duchi2008efficient}.

In our setup we have used the handwritten digits data set MNIST, obtained from \cite{lecun0000website}, which contains $60,000$ training images and $10,000$ testing images, approximately evenly distributed among its $10$ classes (see \cite{lecun1998gradient} for details). We have only used the first five digits.

We present a comparison of the performance, measured by the $\ell_1$ distance between the true $\pi$ and the estimate $\hat{\pi}$, of the sparse KME compared to the full KME. We have done this for different values of $\pi$, meaning different locations of $\pi$ inside the simplex. To do so, we sampled $\pi$ from the simplex using the Dirichlet distribution with different concentration parameter $\omega$. As a reminder to the reader, a small value of $\omega$ implies sparse values of $\pi$ are most probable, $\omega=1$ means any value of $\pi$ is equally probable, and $\omega>1$ means values of $\pi$ for which all its entries are of similar value are most probable. We  varied $\omega$ over the set $\set{.1, .2 ,\dots, 3.1}$.

We have split the data in two and used the first half to estimate the kernel bandwidth through the following process. We first sample a true $\pi$, then we construct the KME and pick the bandwidth $\sigma$ which minimizes $\norm{\pi - \hat{\pi}}_{\ell_1}$. We performed the search on $\sigma$ by using Matlab's function {\em fminbnd}. For the SKM case we allowed for 200 iterations, while for the full KME case we only allowed for 20 iterations since the computation time is expensive. We have used the Gaussian kernel, to create the sparse KME of the $\ell^{th}$ distribution, with sparsity level of $k_\ell = 3\sqrt{n_\ell}$, where $n_\ell$ is the size of the available sample from distribution $\ell$. Since the $\alpha$ coefficients depend on $\sigma$, and for each set $\sI^{(\ell)}$ we perform a search over several values of $\sigma$, we did not compute $\alpha$ iteratively as we constructed $\sI^{(\ell)}$. Instead, once the construction of $\sI^{(\ell)}$ was finished, we used the preconditioned conjugate gradient method to obtain $\alpha$.

Once $\sigma$ was estimated, we then accessed the second half of the data to test the performance for both the SKM and the full KME for different values of $\omega$. The results are shown in Fig. \ref{fig:dirich}. We have also plotted for perspective a ``blind" estimation of $\pi$, which uniformly at random picks a vector $\hat{\pi}$. A comparison of the computation times for the sparse KME and the full KME is shown in Table \ref{tab:classproportion}, where we have averaged over all values of $\omega$.

Notice from Table \ref{tab:classproportion} that, in the SKM case, the estimation of $\sigma$ takes about the same time as the computation of $\hat{\pi}$. This is due to the fact that the main bottleneck of the algorithm is the computation of the set $\sI$ which is independent of $\sigma$. In the case of finding an optimal $\sigma$, we applied ten times more iterations than in the full KME case, while keeping the process ten times faster.

\begin{figure}[htp]
  \centering
    \includegraphics[width=1\textwidth]{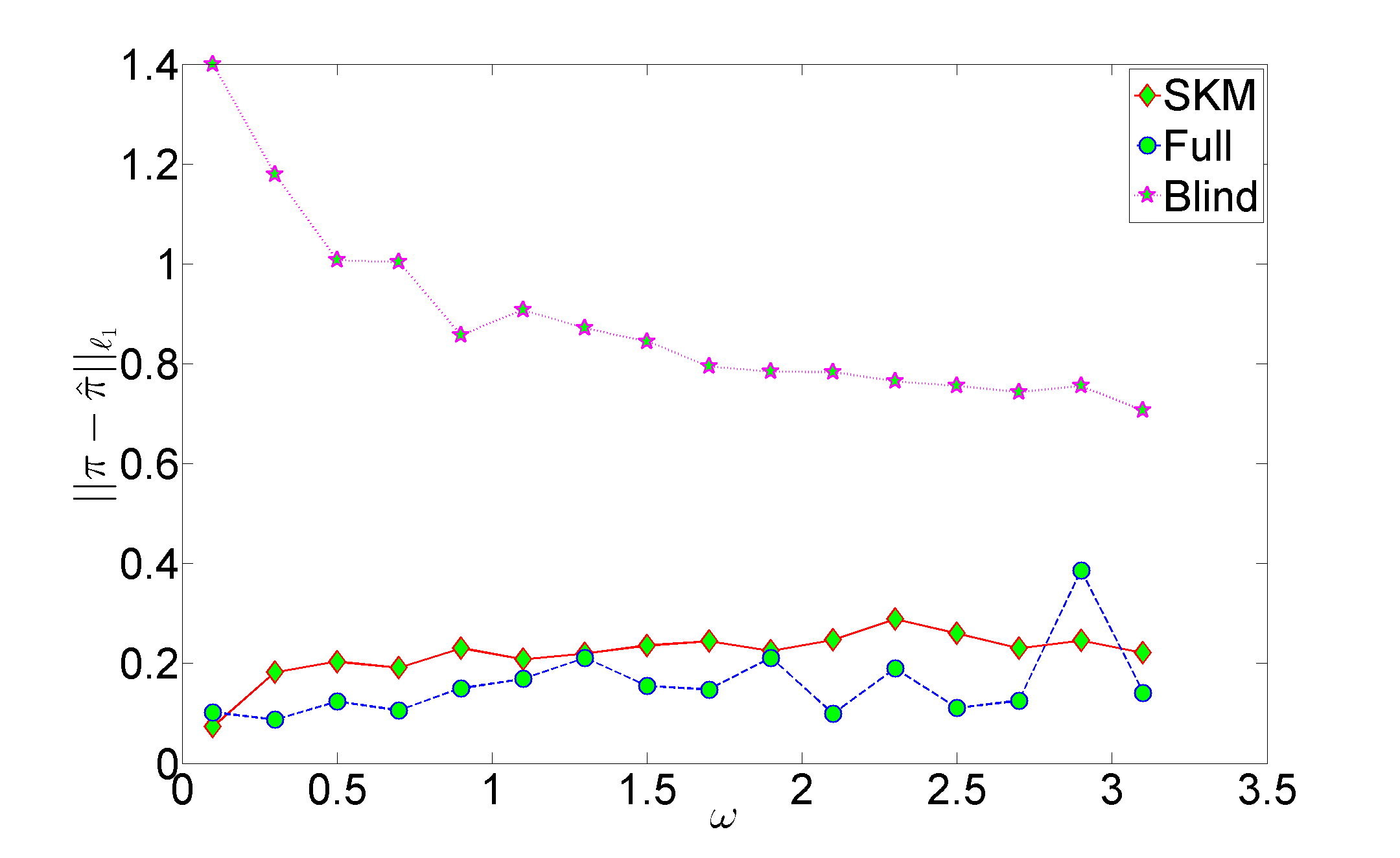}
    \caption{Class Proportion Estimation. $\ell_1$ error of estimated proportions over a range of concentration parameters.}
\label{fig:dirich}
\end{figure}

\begin{table}[htp]
  \centering
      \caption{Computation times for both full and sparse KME, averaged over all values of $\omega$.}
    \begin{tabular}{llll}
    \hline
	& $\sigma$ estimation  & $\hat{\pi}$ computation & Total \\
    \hline
    Full  & 481s & 26s & 8.45mins \\
    SKM & 47.5s & 48.6s & 1.6mins \\
    \hline
    \end{tabular}%
  \label{tab:classproportion}%
\end{table}

\subsection{Mean-Shift Clustering} \label{sec:meanshift}
We have based this experiment on the mean-shift algorithm as described in \cite{comaniciu2002mean}. This algorithm is used in several image processing tasks and we will use it in the context of image segmentation. The goal is to form a clustering of the image pixels into different segments.

Each pixel is represented by a $5$-dimensional vector ($3$ dimensions to describe color, and $2$ for the position in the image), and the distribution of these feature vectors is estimated by the KDE. Denote the image pixels as $\set{x_i}_{i=1}^n$, $x_i \in \bbR^5$. The mean-shift algorithm shifts each point lying on the surface of the density closer to its closest peak (mode). Given a starting point $x$, the algorithm iteratively shifts $x$ closer to its mode until the magnitude of the shift is smaller than some quantity $\gamma$. The shift exerted on $x$ at each iteration requires the computation of the gradient of the KDE at the current position, making mean-shift computationally expensive. Denote the shifted points as $\set{y_i}_{i=1}^n$. Once all points are shifted close to the different modes, then any clustering algorithm can be performed to find the clusters. A clustering algorithm is described in \cite{comaniciu2002mean}, based on merging the modes' neighborhoods which are close. We used a code following these guidelines found at \cite{finkston0000website}, slightly modified by increasing the distance used for modes' neighborhoods to merge.

In our experiments we used a $500 \times 487$ image of a painting by Piet Mondrian ({\em Composition A}), and compared our algorithm with the full density estimation case. We chose $k_{max}$ to be the largest integer smaller than $\sqrt{n}$ and we have used the method for auto-selecting $k_0$ outlined in Section \ref{sec:autoselect}, with $\epsilon=10^{-8}$. We have used the Gaussian kernel and set the bandwidth according to Equation (18) in \cite{chen2014enhanced}, which is specifically suggested for mode-based clustering. We compare the SKM approach to a method based on Locality Sensitive Hashing (LSH, see \cite{gionis1999similarity, andoni2006near}). This method finds for each point and with high probability its nearest neighbors, it then approximates the KDE locally by only using the effect from such neighbors. We chose 5 nearest neighbors and to implement LSH we used the Matlab version of LSH available at \cite{shakhnarovich0000website} (we have used the e2lsh scheme with three hash tables per picture). See \cite{shakhnarovich0000website, andoni0000website} for details on LSH.

We present two indicators to evaluate the performance between the clustering resulting from the full KDE and that resulting from the approximate KDE. In the following, let $\mathcal{B}$ be used to indicate that the full kernel density estimate has been used, while $\mathcal{A}$ indicates either the SKM or the LSH approaches. With a slight abuse of notation, let $\mathcal{A}$ and $\mathcal{B}$ also indicate their resulting clusterings.

{\em Discrepancy Index}. Our first performance measure, which we call the discrepancy index $d_i$, is somehow intuitive, and it describes the ratio of the number of vectors $x_\ell$ which the approximate methods shifted by more than $\delta$ away from their full method counterpart. $\delta$ is here some tolerance threshold, which we have set to three times the kernel bandwidth. More precisely, if $\set{x_\ell}_{\ell=1}^n$ indicate the picture pixels and $y_\ell^\mathcal{A}$, $y_\ell^\mathcal{B}$ are the shifted versions of $x_\ell$ according to density estimation methods $\mathcal{A}$ and $\mathcal{B}$ respectively, then
$$
d_i(\mathcal{A},\mathcal{B}) = \frac{1}{n}\sum_\ell \ind{\norm{y_\ell^\mathcal{A}-y_\ell^\mathcal{B}}>\delta}(x_\ell).
$$

{\em Hausdorff Distance}. The second performance measure, which describes the Hausdorff distance between clusterings, was obtained from \cite{chacon2014population} and is denoted by ${d}_H$. To define the Hausdorff distance, let $P$ be a distribution on $\rd$ (in our case, $P$ is the distribution of the image pixels on $\bbR^5$). Furthermore, let $\mathcal{X}$ be the set of subsets of $\rd$ such that the distance between two sets $A$ and $B$ is $\rho(A,B):=P(A\Delta B)$, where $\Delta$ is the symmetric difference (to be precise, we deal with equivalence classes, where two sets $A$ and $B$ are equivalent if $\rho(A,B)=0$). Notice $\mathcal{X}$ is a metric space. Let $\mathcal{B} \subset \mathcal{X}$, and define $\rho(A,\mathcal{B}):=\min_{B\in \mathcal{B}}\rho(A,B)$. We interpret a subset $\mathcal{A}$ of $\mathcal{X}$ as a clustering, and an element $A$ in $\mathcal{X}$ as a cluster. The Hausdorff distance between two clusterings is
$$
d_H(\mathcal{A},\mathcal{B}) = \max \set{\underset{A\in\mathcal{A}}{\max} \,\,\, \rho(A,\mathcal{B}), \,\, \underset{B\in\mathcal{B}}{\max} \,\,\, \rho({B,\mathcal{A}})}.
$$

In words, $d_H$ measures the furthest distance between elements of $\mathcal{A}$ to the clustering $\mathcal{B}$ and elements of $\mathcal{B}$ to the clustering $\mathcal{A}$. That is, the less overlap between clusters of different clusterings, as measured by $P$. Since we don't have access to $P$, the empirical version of $d_H$ proposed in \cite{chacon2014population} is obtained by replacing $P$ for the empirical probability measure. Letting $\widehat{\rho}(A,\mathcal{B}) := \min_{B\in \mathcal{B}} \frac{1}{n} \sum_{i=1}^{n}\ind{A\Delta B}(x_i)$, we have
$$
\widehat{d}_H(\mathcal{A},\mathcal{B}) = \max \set{\underset{A\in\mathcal{A}}{\max} \,\,\, \widehat{\rho}(A,\mathcal{B}), \,\, \underset{B\in\mathcal{B}}{\max} \,\,\, \widehat{\rho}({B,\mathcal{A}})}.
$$
We use this latter quantity to measure the SKM performance.

The results are presented in Table \ref{tab:meanshift}. In the table $\mathcal{B}$ indicates the full kernel density estimate has been used, $\mathcal{A}_{SKM}$ indicates the $k$-center based algorithm and $\mathcal{A}_{LSH}$ the LSH setup. Note that both the SKM and the LSH approach present significant computational advantages. The SKM approach, however, manages to be faster while incurring half the discrepancy of the LSH and about the same Hausdorff distance.

\begin{table}[htp]
  \centering
   \caption{
Time and Performance Comparison for Mean Shift algorithm.}
    \begin{tabular}{llllll}
    \hline
    & \multicolumn{3}{c}{Time} & \multicolumn{2}{c}{Performance} \\
    \hline
	& Preparation  & Mean Shift & Total & $d_i$($\cdot,\mathcal{B}$)  & $\widehat{d}_H$($\cdot,\mathcal{B}$) \\
    \hline
    $\mathcal{B}$  & 0 & 4hrs & 4hrs & 0 & 0 \\
    $\mathcal{A}_{SKM}$ & 3.26mins & 57s & 4.2mins & 0.018 & 0.021 \\
   $\mathcal{A}_{LSH}$ & 14s & 4.2mins & 4.4mins & 0.034 & 0.016 \\
    \hline

    \end{tabular}
  \label{tab:meanshift}
\end{table}

\subsection{Other Simulations}
\label{sec:simulation}

Unlike other methods for approximating a sum of kernels, the sparse approximation strategy proposed in this paper has the advantage that the resulting approximation can be a valid density if the $\alpha_i$'s are set to satisfy $\alpha_i\geq 0$ and $\sum_i \alpha_i = 1$ . Therefore, we also evaluate the performance of the proposed sparse approximation according to the KL divergence, a common metric between distributions whose arguments must be density functions. Notice in particular that other KDE approximation methods like the Improved Fast Gauss Transform and the LSH-based approach described in Section \ref{sec:meanshift} are not applicable since they don't return valid densities.

For 11 distinct benchmark data sets, listed in Table \ref{tab:Dqp}, we computed the KL divergences $D(\zbar \| z_\sI)$ and $D(z_\sI \| \zbar)$ between the sparse and the full kernel mean. We used the auto-selection scheme proposed in Section \ref{sec:autoselect}, and projected the resulting $\alpha$ onto the simplex to ensure we have a valid probability distribution. We have chosen a Gaussian kernel and used the Jaakkola heuristic \cite{jaakkola1999using} to compute the bandwidth. To place the performance of our approximation in perspective, we have also computed the KL divergences for a sparse approximation based on choosing the set $\sI$ uniformly at random. We have performed the Wilcoxon rank test \cite{wilcoxon1945individual} to determine if there is a significant advantage of the SKM. The test for both the case D($\bar{z} \| z_\sI$) and the case D($z_\sI \| \bar{z}$) yields a $p$-value of 0.0186, favoring the SKM method. The results are shown in Table \ref{tab:Dqp}.

To further illustrate the performance of SKM, we look at the error quantities $E_{\abs{\sI}} = \norm{\zbar - z_\sI}^2 - \norm{\zbar}^2$ (see Section \ref{sec:autoselect}) as the size of $\sI$ increases.  Fig. \ref{fig:kvsrand} shows a plot of $E_{\abs{\sI}}$ against the size of $\sI$ for the banana data set. As a baseline, we have plotted alongside the same error for an approximation based on choosing the set $\sI$ uniformly at random. Since we want to explore how fast $ \norm{\zbar - z_\sI}^2$ approaches zero, we allowed $k_{max}=n$ but used the auto-selection scheme to stop at an earlier $k_0$ with tolerance threshold $\epsilon=10^{-9}$. We averaged 100 times and, at each iteration, we completed the graph by letting $E_{\abs{\sI}}=E_{k_0}$ for $\abs{\sI}>k_0$. The average SKM run stopped at $k_0$=197, and the random sampling comparison at $k_0=240$. The random approximation shows an initial advantage because it is more likely to pick elements from dense areas, which for small values of $\sI$ represents better the full distribution. However, as the size of $\sI$ increases the fine structure (e.g., the distribution tails) is better captured by SKM, since the $k$-center algorithm picks points far apart from each other.

\begin{table}[htp]
  \centering
  \caption{
Values of D($\bar{z} \| z_\sI$) and D($z_\sI \| \bar{z}$) for different data sets.}
    \begin{tabular}{ccccc}
    \hline
    & \multicolumn{2}{c}{D($\bar{z} \| z_\sI$)} & \multicolumn{2}{c}{D($z_\sI \| \bar{z}$)} \\
    \hline
    & Random & SKM & Random & SKM \\
    \hline
banana &    0.092597 & 0.001805 & 0.129183 & 0.001613 \\
image &    0.451205 & 0.041305 & 0.212585 & 0.061584 \\
ringnorm&    0.003983 & 0.031736 & 0.009253 & 0.02853 \\
breast-cancer&    0.358253 & 0.002546 & 0.345895 & 4.56E-05 \\
heart&    0.001918 & 6.35E-16 & 0.005228 & 2.91E-16 \\
thyroid&    0.177317 & 0.000594 & 0.034616 & 0.000289 \\
diabetes&    0.031366 & 0.005474 & 0.014635 & 0.000102 \\
german&    0.008711 & 0.003855 & 0.008742 & 0.00203 \\
twonorm&    0.000131 & 0.000243 & 4.59E-05 & 0.000372 \\
waveform&    0.011473 & 0.000177 & 0.015064 & 0.000404 \\
iris&    0.043924 & 0.000395 & 0.022519 & 0.000104 \\
    \hline
    \end{tabular}%
  \label{tab:Dqp}
\end{table}%

\begin{figure}[htp]
  \centering
    \includegraphics[width=1\textwidth]{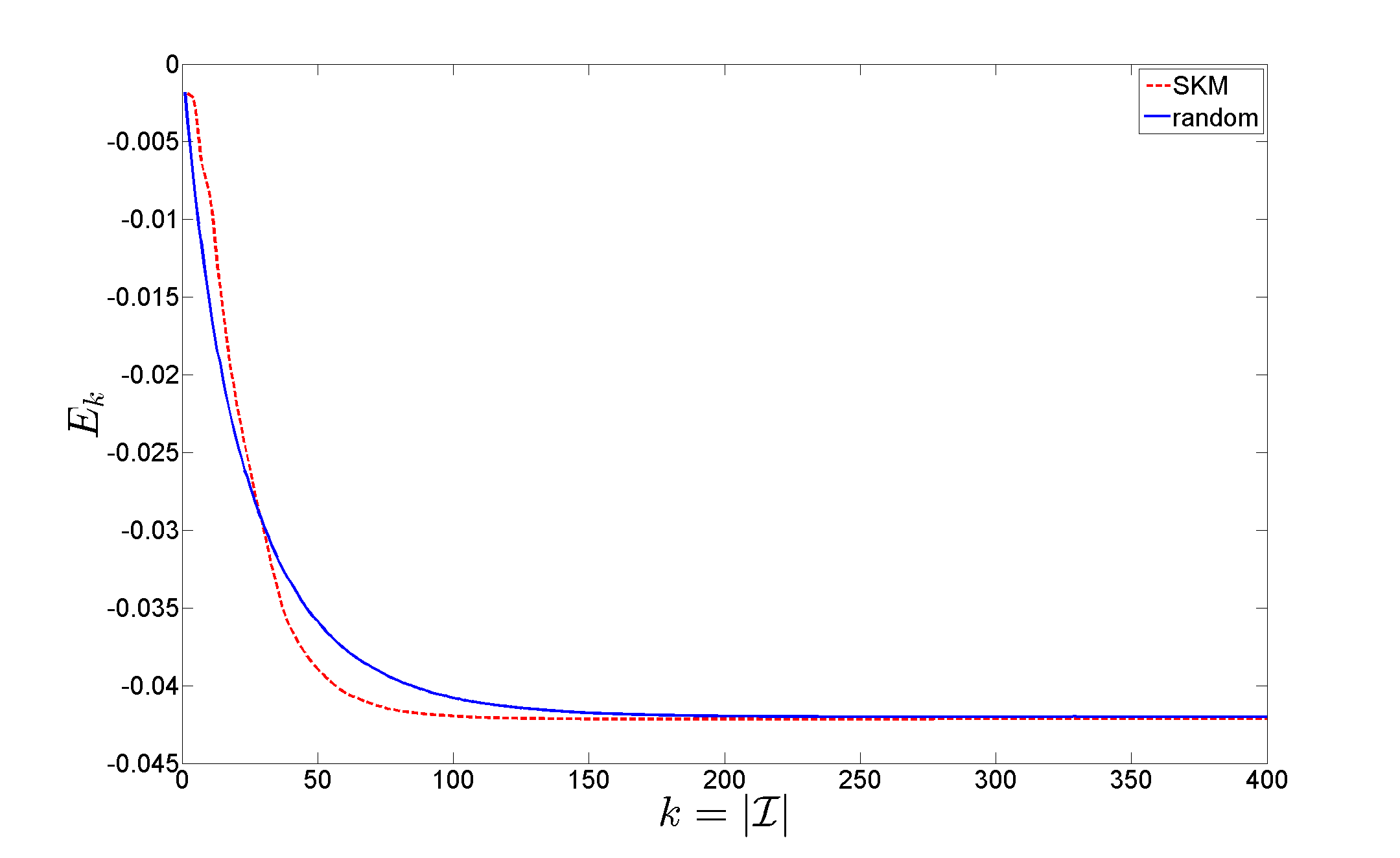}
\caption{ Comparison of $E_{\abs{\sI}}$ between the random algorithm and the $k$-center algorithm for the banana data set.}
\label{fig:kvsrand}
\end{figure}

\section{Conclusion}
We have provided a method to rapidly and accurately build a sparse approximation of a kernel mean. We derived an incoherence based bound on the approximation error and recognized that, for radial kernels, its minimization is equivalent to solving the $k$-center problem on the data points. If desired, our construction of the sparse kernel mean may be slightly modified to provide a valid density function, which is important in some applications. Furthermore, the algorithm works for both kinds of kernel means: the KDE and the KME. Our method also naturally lends itself to a sparsity auto-selection scheme.

We showed its computational advantages and its performance qualities in three specific applications. First, Euclidean embedding of distributions (for both KDE and KME), in which, for the KDE case, a valid density is needed to compute the KL divergence. Second, class proportion estimation (for the KME), which presents the amortization advantages of the SKM approach, in this case with respect to the bandwidth $\sigma$. Finally, mean-shift clustering (for the KDE), in which with less computation time than the LSH-based approach, it performs better with respect to the discrepancy index and similar with respect to the Hausdorff distance. In most instances the proposed sparse kernel mean method has shown to be orders of magnitude faster than the approach based on the full kernel mean.

\section*{Acknowledgments}
The authors thank Lloyd Stoolman of the University of Michigan Department of Pathology for providing the de-identified flow cytometry data set. This work was supported in part by NSF Awards 0953135, 1047871, 1217880, and 1422157.
\bibliographystyle{IEEEtran}
\bibliography{efrenbib}

\end{document}